\documentclass[10pt]{article} 

\usepackage[preprint]{tmlr}

\usepackage{amsmath,amsfonts,bm}

\def\eqref#1{equation~\ref{#1}}

\def\1{\bm{1}}

\DeclareMathAlphabet{\mathsfit}{\encodingdefault}{\sfdefault}{m}{sl}
\SetMathAlphabet{\mathsfit}{bold}{\encodingdefault}{\sfdefault}{bx}{n}

\usepackage{hyperref}
\usepackage{url}
\usepackage{lmodern}
\usepackage{microtype}
\usepackage{amsmath,amssymb,amsfonts,amsthm,mathtools}
\usepackage{bm}
\usepackage{bbm}
\usepackage{mathrsfs}
\usepackage{enumitem}
\usepackage{booktabs}
\usepackage{multirow}
\usepackage{array}
\usepackage{xcolor}
\usepackage[nameinlink,capitalise,noabbrev]{cleveref}
\usepackage{graphicx}
\usepackage{tikz-cd}
\usepackage{algorithm}
\usepackage{algpseudocode}
\usepackage{caption}
\usepackage{subcaption}
\usepackage{setspace}
\usepackage{comment}
\usepackage{float}              
\usepackage[section]{placeins}  
\usepackage[table]{xcolor}

\newtheorem{theorem}{Theorem}[section]
\newtheorem{proposition}[theorem]{Proposition}

\newtheorem{corollary}[theorem]{Corollary}
\newtheorem{definition}[theorem]{Definition}

\newtheorem{remark}[theorem]{Remark}

\title{Semantic Sections: An Atlas-Native Feature Ontology for Obstructed Representation Spaces}

\author{\name Hossein Javidnia \email hossein.javidnia@dcu.ie \\
      \addr School of Computing\\
      Dublin City University, Ireland}

\def\month{03}  
\def\year{2026} 
\def\openreview{\url{https://openreview.net/forum?id=XXXX}} 

\begin{document}

\maketitle

\begin{abstract}
Recent interpretability work often treats a feature as a single global direction, dictionary atom, or latent coordinate shared across all contexts.
However, recent atlas-based work has shown that this ontology can fail in obstructed representation spaces: proxy shearing and nontrivial holonomy imply that locally coherent meanings need not assemble into one globally consistent feature \cite{javidnia2026gaugetheory}.
In this paper, we introduce an atlas-native replacement object, the \emph{semantic section}: a transport-compatible family of local feature representatives defined over a context atlas.
We formalize semantic sections, prove that tree-supported propagation is always pathwise realizable, and show that cycle consistency is the decisive criterion for genuine globalization.
This yields a principled distinction between \emph{tree-local}, \emph{globalizable}, and \emph{twisted} sections, with twisted sections capturing locally coherent but holonomy-obstructed meanings.

We then develop a discovery-and-certification pipeline for semantic sections in frozen LLM atlases based on seeded propagation, synchronization across overlaps, defect-based pruning, cycle-aware taxonomy, and deduplication.
Across layer-16 atlases for \texttt{Llama~3.2~3B~Instruct} \cite{grattafiori2024llama3,meta2024llama32card}, \texttt{Qwen~2.5~3B~Instruct} \cite{qwen2024qwen25,qwen2024qwen25card}, and \texttt{Gemma~2~2B~IT} \cite{gemmateam2024gemma2,google2024gemma2card}, we find nontrivial populations of semantic sections, including cycle-supported globalizable and twisted regimes after deduplication.
Gemma exhibits the richest cycle-supported structure, while Qwen provides a useful control regime in which tree-local coherence dominates.
The empirical picture is robust to threshold variation and, for Llama, persists across multiple layers.

Most importantly, we show that semantic identity is not recovered by raw global-vector similarity.
Even for certified globalizable sections, cross-chart signed cosine similarity between local representatives is typically low, and raw similarity baselines recover only a small fraction of true within-section pairs, often collapsing to zero at moderate thresholds.
By contrast, section-based identity recovery is perfect on certified supports.
These results support the claim that, in obstructed regimes, a feature is better modeled as a semantic section than as a single global vector. 

\smallskip
\noindent\textbf{Code:} \url{https://github.com/hosseinjavidnia/gauge_superposition}
\end{abstract}

\section{Introduction}
\label{sec:introduction}

Much of modern mechanistic interpretability is built around a simple ontological assumption:
a \emph{feature} is a single object that can be represented globally, for example as one direction in activation space, one sparse dictionary atom, one latent coordinate, or one neuron-like axis \cite{olah_feature_2017,bau2017networkdissection,kim2018tcav,olah2020zoom,elhage2022toy,cunningham2023sparse,bricken2023monosemanticity}.
This assumption is powerful because it turns interpretability into a search problem over globally shared objects.
If a feature is globally well-defined, then one can hope to identify it once, name it once, and track it everywhere.

However, this picture can fail when neural representations are only locally compatible.
Recent atlas-based work has shown that local feature structure can organize into a nontrivial context atlas with learned transition maps between overlapping charts, and that these transitions can exhibit both \emph{proxy shearing} and \emph{holonomy} \cite{javidnia2026gaugetheory}.
In such regimes, there need not exist one globally consistent feature ontology.
A candidate meaning may look coherent in each local chart while still failing to glue into a single global direction.
This creates a conceptual gap in current interpretability practice:
if global features fail, what is the correct replacement object?

This paper proposes an answer.
We argue that, in obstructed representation spaces, the right primitive is not a single global vector but a \emph{semantic section}:
a transport-compatible family of local feature representatives defined over a context atlas.
Intuitively, a semantic section treats a feature as something that may have different local representatives in different charts, provided these representatives are related by the atlas transport structure.
The central question is therefore no longer whether one vector persists globally, but whether a family of local representatives can be propagated, synchronized, and made cycle-consistent across the atlas.

This shift has both theoretical and empirical consequences.
On the theory side, it replaces the global-feature ontology with a local-to-global one.
We show that tree-supported propagation is always pathwise realizable, but that genuine semantic globalization on cyclic supports is controlled by loop consistency.
This yields a natural distinction between three qualitatively different regimes:
\emph{tree-local} sections, which are coherent only in the weak pathwise sense;
\emph{globalizable} sections, which remain coherent around loops; and
\emph{twisted} sections, which are locally coherent but fail to globalize because loop transport returns a rotated representative.
In this way, holonomy becomes not merely a diagnostic of atlas geometry, but an obstruction to single-valued semantic identity.

On the empirical side, the semantic-section view leads to a new discovery problem.
Rather than fitting one global feature vector, we begin from a local seed representative, propagate it through the chart graph, synchronize multiply reached charts, prune by transport defect, and classify the resulting object by its cycle structure and loop behavior.
This yields a discovery-and-certification pipeline for semantic sections in frozen LLM atlases.
Applied to layer-$16$ atlases for \texttt{Llama~3.2~3B~Instruct} \cite{grattafiori2024llama3,meta2024llama32card}, \texttt{Qwen~2.5~3B~Instruct} \cite{qwen2024qwen25,qwen2024qwen25card}, and \texttt{Gemma~2~2B~IT} \cite{gemmateam2024gemma2,google2024gemma2card}, the pipeline finds nontrivial populations of globalizable, twisted, and tree-local sections after deduplication.
Gemma exhibits the richest cycle-supported section structure, Llama realizes the full taxonomy more sparsely, and Qwen provides a useful control regime in which tree-local coherence dominates.

The most important empirical result is that semantic identity is not recovered by raw global-vector similarity.
Even when a discovered object is strongly coherent as a semantic section, its local representatives are typically not close as one shared vector across charts.
In aggregate, section-based identity recovery is perfect on certified supports, while raw signed-cosine similarity recovers only a small fraction of true within-section pairs and often collapses to zero at moderate thresholds.
For example, across $14$ deduplicated globalizable sections in the Gemma layer-$16$ atlas, comprising $1221$ true within-section chart pairs, raw similarity recovers only $19.57\%$ of pairs at threshold $0.3$, only $2.62\%$ at threshold $0.5$, and none at threshold $0.7$.
These results suggest that, in obstructed regimes, the global-vector ontology fails at the basic task of semantic identity recovery, whereas semantic sections succeed because they are the correct atlas-native object.

The contribution of the paper is therefore not merely a new diagnostic or a new clustering heuristic.
It is a proposal for a different feature ontology.
Our claim is not that global features are never useful, but that they are not the right primitive in the presence of representation-theoretic obstruction.
When local meanings fail to assemble into one globally stable vector, they should instead be modeled as transport-compatible local sections.

\paragraph{Contributions.}
The main contributions of this paper are:
\begin{enumerate}[leftmargin=2em]
\item We introduce \emph{semantic sections} as an atlas-native replacement for the global-feature ontology in obstructed representation spaces.
\item We formalize tree-local, globalizable, and twisted section regimes, and identify cycle consistency and holonomy as the key local-to-global criteria.
\item We develop a practical discovery-and-certification pipeline for semantic sections in frozen LLM atlases.
\item We provide cross-model evidence that semantic sections are empirically discoverable, structurally heterogeneous, and robust to threshold and layer variation.
\item We show that semantic identity is recovered by sections but not by raw global-vector similarity, supporting the claim that features in obstructed regimes are better modeled as semantic sections than as single global vectors.
\end{enumerate}

The remainder of the paper proceeds as follows.
Section~\ref{sec:semantic-sections} introduces the semantic-section formalism and its local-to-global theory.
Section~\ref{sec:discovery-certification} describes the discovery pipeline and empirical protocol.
Section~\ref{sec:experiments} presents the empirical results, beginning with a cross-model census of section types and then analyzing geometric profiles, representative exemplars, identity recovery, and robustness across layers and thresholds.

\section{Related Work}
\label{sec:related-work}

\paragraph{Mechanistic interpretability and global feature ontologies.}
A large body of mechanistic interpretability work treats internal model structure as decomposable into globally meaningful units, such as neurons, directions, circuits, or learned latent features \cite{olah_feature_2017,bau2017networkdissection,kim2018tcav,olah2020zoom,olsson2022induction,wang2022ioi,elhage2022toy,cunningham2023sparse,bricken2023monosemanticity}.
This perspective has been highly productive, especially in settings where one can identify reusable directions or sparse atoms that appear stable across many contexts.
Our work is continuous with this tradition in its goal of isolating interpretable semantic structure, but differs in the ontological object it assigns semantic identity to.
Rather than assuming that one globally shared vector is the right primitive, we ask what replaces that primitive when representation geometry is only locally compatible.

\paragraph{Superposition, sparse decompositions, and monosemantic features.}
Recent work on superposition, circuits, and sparse autoencoders has sharpened the idea that neurons are often not the correct unit of analysis, and that more interpretable latent structure can instead be recovered as directions, mechanisms, or dictionary atoms in activation space \cite{olah2020zoom,olsson2022induction,wang2022ioi,elhage2022toy,cunningham2023sparse,bricken2023monosemanticity}.
These approaches provide strong evidence that feature-like structure can often be recovered by overcomplete or sparse decompositions.
However, they typically still assume a \emph{global} feature ontology: once identified, a feature is treated as one object shared across contexts.
The semantic-section framework proposed here is compatible with the motivation behind these methods, but departs from their default ontology in obstructed regimes.
Our claim is that when transition structure between local charts is nontrivial, semantic identity may be represented not by one global atom but by a transport-compatible family of local representatives.

\paragraph{Geometric and local-to-global views of representation.}
Our work is also related to broader geometric perspectives on representation learning, in which structure is understood through local linearization, manifold organization, intrinsic dimension, or coordinate-dependent descriptions rather than through one preferred global basis \cite{bengio2013representation,rifai2011manifold,fefferman2016manifold,ansuini2019intrinsic,cohen2020separability}.
The closest precursor is the recent atlas-based obstruction framework of \cite{javidnia2026gaugetheory}, which models neural representations as a context atlas with local charts and overlap transports, and identifies proxy shearing and holonomy as obstructions to globally consistent feature structure.
The present paper builds directly on that framework.
Where \cite{javidnia2026gaugetheory} diagnoses when a global ontology can fail, we ask what semantic object should replace it.
Our answer is the semantic section: an atlas-native local-to-global object whose coherence is defined by transport compatibility and whose globalization is controlled by cycle consistency.

\paragraph{What is new here.}
This paper is therefore not merely another feature-discovery method, circuit analysis, or sparse decomposition technique.
Its main novelty is ontological.
We introduce semantic sections as the appropriate replacement for global feature vectors in obstructed representation spaces, formalize the distinction between tree-local, globalizable, and twisted sections, and provide empirical evidence that semantic identity can be recovered by sections even when raw global-vector similarity fails.
In this sense, the paper extends atlas-based representation theory from a geometric diagnosis of obstruction to a semantic theory of feature identity.

\section{Semantic sections and atlas-native features}
\label{sec:semantic-sections}

Recent atlas-based work showed that a global feature ontology can fail in obstructed representation spaces: local feature structure may be compatible chartwise while still failing to assemble into one globally consistent semantic object because of proxy shearing and nontrivial holonomy \cite{javidnia2026gaugetheory}.
We now introduce the replacement object for such regimes: a \emph{semantic section}, namely a transport-compatible family of local feature representatives defined over a context atlas.

\subsection{Semantic sections}

Let $G=(V,E)$ be the context-chart graph of an atlas in the sense of \cite{javidnia2026gaugetheory}.
For each chart $c\in V$, let $F_c \cong \mathbb{R}^k$ denote its local feature space, and for each oriented edge $u\to v$ let
\[
T_{vu}:F_u\to F_v
\]
denote the transport map induced on the overlap from chart $u$ to chart $v$.
In practice, $T_{vu}$ may be taken to be an orthogonal proxy transport, the polar factor of a learned edge map, or another fixed transport operator derived from the atlas construction.

The basic semantic idea is that a feature need not be represented by one globally shared vector.
Instead, it may be represented by a family of local vectors, one per chart, provided these local representatives are compatible under transport.

\begin{definition}[Semantic section]
\label{def:semantic-section}
Let $H=(V_H,E_H)$ be a connected subgraph of $G$.
A \emph{semantic section} on $H$ is a family
\[
s=\{s_c\}_{c\in V_H}, \qquad s_c\in F_c\setminus\{0\},
\]
considered up to independent positive rescaling and sign in each chart.
For an oriented edge $u\to v$ in $E_H$, define the edge defect
\[
\delta_{vu}(s)
:=
\inf_{\sigma\in\{\pm 1\}}
\left\|
\frac{T_{vu}s_u}{\|T_{vu}s_u\|_2}
-
\sigma \frac{s_v}{\|s_v\|_2}
\right\|_2,
\]
whenever $T_{vu}s_u\neq 0$.
We say that $s$ is \emph{$\varepsilon$-coherent} on $H$ if
\[
\delta_{vu}(s)\le \varepsilon
\qquad
\text{for every oriented edge }u\to v\text{ in }E_H.
\]
\end{definition}

The quotient by sign reflects the fact that many local feature representations are semantically identified up to orientation, while the quotient by positive rescaling reflects that semantic identity is directional rather than norm-dependent.
Thus a semantic section is not one global vector, but a locally defined family of representatives whose identity is tracked by transport across overlaps.

\subsection{Tree-local extension}

The first basic fact is that on a tree, local compatibility is pathwise trivial.

\begin{proposition}[Tree extension]
\label{prop:tree-extension}
Let $H=(V_H,E_H)$ be a connected tree, let $r\in V_H$ be a root, and let $s_r\in F_r\setminus\{0\}$ be a seed representative.
Assume that along every edge of the rooted tree, the transported representative is nonzero.
Then there exists a normalized section $s=\{s_c\}_{c\in V_H}$, unique up to an overall sign choice at the root, such that for every parent-child edge $u\to v$ in the rooted tree,
\[
\frac{s_v}{\|s_v\|_2}
=
\sigma_{vu}\,
\frac{T_{vu}s_u}{\|T_{vu}s_u\|_2}
\qquad
\text{for some }\sigma_{vu}\in\{\pm 1\}.
\]
In particular, every seed extends uniquely along the tree by recursive transport and normalization.
\end{proposition}

\begin{proof}
Because $H$ is a tree, every chart $c\in V_H$ is connected to the root $r$ by a unique simple path.
Define $s_c$ by transporting $s_r$ recursively along that path and normalizing at each step, choosing a sign on each edge as needed.
Existence follows from the nonvanishing assumption on propagated representatives.
Uniqueness up to the root sign follows from uniqueness of paths in a tree.
\end{proof}

\begin{remark}
\label{rem:tree-local-weak}
Proposition~\ref{prop:tree-extension} shows that tree-supported sections are weak evidence for semantic globalization.
On a tree there are no independent cycles, so there is no possibility of comparing incompatible transport routes.
Tree-local coherence therefore tests only pathwise propagation, not genuine local-to-global consistency.
\end{remark}

\subsection{Cycle-supported globalization}

The situation changes on cyclic supports, where multiple transport routes may connect the same charts.

\begin{definition}[Globalizable section]
\label{def:globalizable-section}
Let $H\subseteq G$ be connected.
A section $s$ on $H$ is \emph{globalizable} if it is $\varepsilon$-coherent on $H$ for small $\varepsilon$ and is approximately path-independent on $H$.
Equivalently, transporting a local representative along two different paths with the same endpoints yields normalized representatives that agree up to sign and small error.
\end{definition}

To formalize the cyclic obstruction, let
\[
\gamma=(c_0,c_1,\dots,c_L=c_0)
\]
be an oriented cycle in $H$.
Define the cycle transport
\[
\mathcal{T}_\gamma
:=
T_{c_0 c_{L-1}}\,T_{c_{L-1} c_{L-2}}\cdots T_{c_2 c_1}T_{c_1 c_0}.
\]
For a base representative $s_{c_0}$, define the cycle defect
\[
\Delta_\gamma(s)
:=
\inf_{\sigma\in\{\pm 1\}}
\left\|
\frac{\mathcal{T}_\gamma s_{c_0}}{\|\mathcal{T}_\gamma s_{c_0}\|_2}
-
\sigma \frac{s_{c_0}}{\|s_{c_0}\|_2}
\right\|_2,
\]
whenever $\mathcal{T}_\gamma s_{c_0}\neq 0$.

The cycle defect measures the failure of loop transport to return the same local representative.
When $\Delta_\gamma(s)=0$, the section closes consistently around the loop $\gamma$ up to sign; when it is positive, loop transport twists the representative away from itself.
Given a seed representative at a root chart, we say that propagation on $H$ is path-independent if any two transport paths from the root to the same chart yield normalized representatives that agree up to sign.
The following result makes precise that cycle consistency is exactly the criterion for path-independent propagation.

\begin{theorem}[Cycle consistency criterion]
\label{thm:cycle-consistency}
Let $H\subseteq G$ be connected, let $r\in V_H$, and let $s_r\in F_r\setminus\{0\}$ be a seed representative.
A normalized propagated section on $H$ is path-independent if and only if
\[
\Delta_\gamma(s)=0
\]
for every cycle $\gamma$ in $H$ based at $r$.
More generally, if $\Delta_\gamma(s)\le \eta$ for every cycle in a cycle basis of $H$, then the discrepancy between normalized representatives obtained along any two paths with the same endpoints is controlled by the accumulated cycle defects.
\end{theorem}

\begin{proof}[Proof sketch]
If propagation is path-independent, then transporting around any closed loop returns the same normalized representative up to sign, so every cycle defect vanishes.
Conversely, if two paths from $r$ to a chart $c$ produce different normalized representatives, then composing one path with the reverse of the other yields a cycle whose defect records that discrepancy.
Thus vanishing cycle defect on all cycles implies path-independence.
The approximate statement follows by expressing arbitrary path discrepancies through a cycle basis decomposition and applying triangle inequalities after normalization.
\end{proof}

\subsection{Holonomy obstruction}

Theorem~\ref{thm:cycle-consistency} upgrades holonomy from a geometric diagnostic to a direct obstruction to semantic globalization \cite{javidnia2026gaugetheory}.

\begin{corollary}[Holonomy obstructs single-valued semantic globalization]
\label{cor:holonomy-obstructs-section}
Let $H\subseteq G$ be connected and let $s$ be an $\varepsilon$-coherent section on $H$.
If there exists a cycle $\gamma\subseteq H$ such that
\[
\Delta_\gamma(s)>0,
\]
then $s$ does not define a single-valued global semantic section on $H$.
In particular, nontrivial loop action on the local representative obstructs interpreting the candidate meaning as one globally consistent feature on that subsystem.
\end{corollary}

\begin{proof}
By Theorem~\ref{thm:cycle-consistency}, single-valued globalization requires vanishing cycle defect on every cycle.
A positive cycle defect therefore obstructs global consistency.
\end{proof}

\subsection{Twisted sections}

\begin{figure}[t]
    \centering
    \includegraphics[width=\columnwidth]{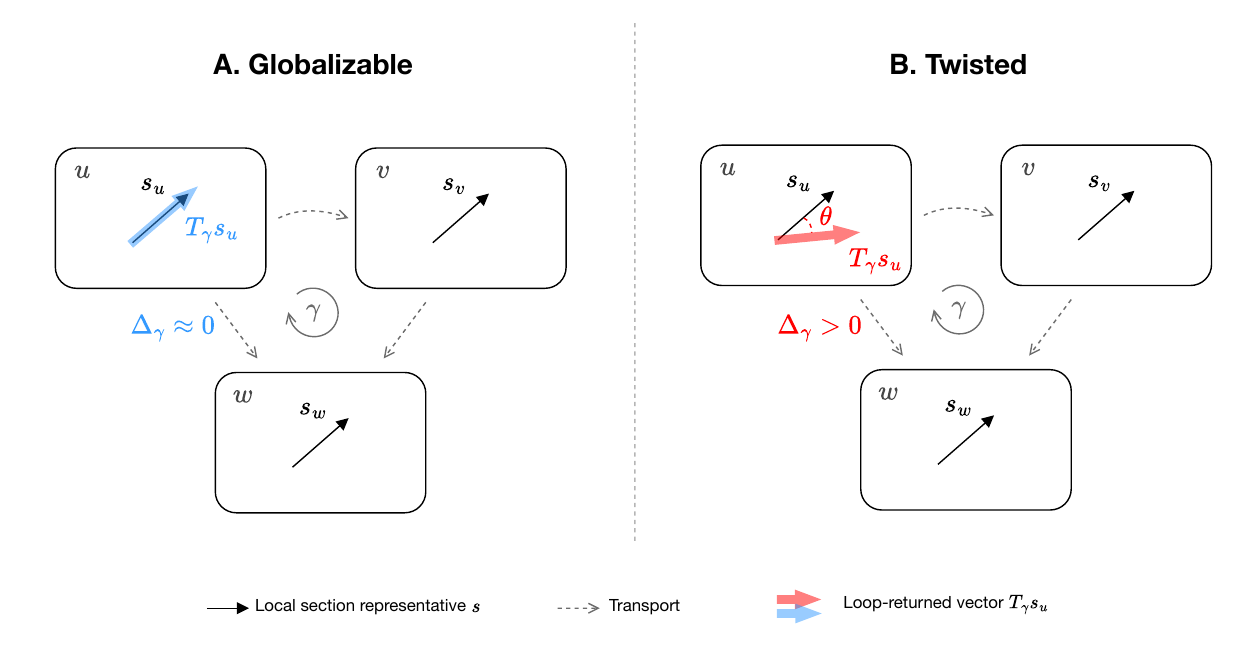}
    \caption{Semantic sections on a cyclic chart subsystem. (A) A globalizable section, where transporting the local representative $s_u$ around a loop $\gamma$ returns a representative $T_\gamma s_u$ that remains nearly aligned with $s_u$, yielding small loop defect $\Delta_\gamma \approx 0$. (B) A twisted section, where the same loop transport returns a rotated representative, so $T_\gamma s_u$ no longer aligns with $s_u$ and the loop defect is nonzero, $\Delta_\gamma > 0$. This illustrates the difference between cycle-consistent semantic globalization and holonomy-obstructed local coherence.}
    \label{fig:globalizable-vs-twisted}
\end{figure}

Figure~\ref{fig:globalizable-vs-twisted} illustrates this distinction schematically: a globalizable section closes consistently around a cycle up to small error, whereas a twisted section remains locally coherent on edges but fails to return to itself under loop transport.

The corollary above motivates a distinct semantic regime.

\begin{definition}[Twisted section]
\label{def:twisted-section}
A section $s$ on a connected cyclic subgraph $H$ is \emph{twisted} if it is locally coherent on edges, in the sense that
\[
\delta_{vu}(s)\le \varepsilon
\qquad
\text{for all oriented edges }u\to v\text{ in }E_H,
\]
but fails the cycle consistency condition on at least one cycle:
\[
\Delta_\gamma(s)>\tau
\]
for some threshold $\tau>0$ and some cycle $\gamma\subseteq H$.
\end{definition}

Twisted sections are therefore not mere failed discoveries or numerical artifacts.
They are locally stable candidate meanings whose failure is specifically local-to-global:
the section remains coherent chartwise, but loop transport induces nontrivial monodromy on the representative.
In this sense, a twisted section is semantically stable on overlaps while remaining globally obstructed.

\section{Discovery and certification of semantic sections}
\label{sec:discovery-certification}

Section~\ref{sec:semantic-sections} defined the target object: a semantic section is a transport-compatible family of local representatives over a chart subsystem.
We now describe the practical pipeline used to discover, certify, classify, and deduplicate such sections in frozen LLM atlases.
The pipeline builds directly on the atlas construction and transport machinery introduced in \cite{javidnia2026gaugetheory}, but changes the semantic target from a putative global feature vector to a locally propagated and cycle-tested section.

\subsection{Seeded propagation on the chart graph}

Fix a chart graph $G=(V,E)$, local feature spaces $\{F_c\}_{c\in V}$, and edge transports $\{T_{vu}\}_{u\to v}$.
Given a seed chart $r\in V$ and a nonzero seed representative $s_r\in F_r$, we construct an initial candidate section by propagating the seed through the chart graph.

For each admissible frontier edge $u\to v$, we form the normalized transported representative
\[
\widetilde s_v^{(u)}
:=
\frac{T_{vu}s_u}{\|T_{vu}s_u\|_2},
\]
whenever $T_{vu}s_u\neq 0$.
Crucially, we do \emph{not} require $\widetilde s_v^{(u)}$ to coincide with a canonical basis axis in chart $v$.
The transported vector itself is the local candidate representative.
This is essential to the semantic-section viewpoint: the object being propagated is an arbitrary local representative glued by transport, not an axis-preserving dictionary atom.

Propagation is restricted by edge-quality filters inherited from the atlas pipeline \cite{javidnia2026gaugetheory}.
In particular, we use overlap size and transport-conditioning statistics to decide which transitions are admissible.
This prevents support growth through numerically degenerate overlaps and ties section discovery to the same local geometry used to build the atlas itself.

\subsection{Synchronization over multiply reached charts}

A chart may be reached along multiple transport routes during propagation.
The resulting incoming representatives need not agree exactly, even when they correspond to the same underlying semantic object.
We therefore perform a synchronization step that aggregates incoming transport evidence at each multiply reached chart.

Suppose chart $v$ receives incoming candidates
\[
\{\widetilde s_v^{(1)},\dots,\widetilde s_v^{(m)}\}.
\]
We choose signs $\sigma_i\in\{\pm 1\}$ to align these representatives and define the synchronized representative
\[
s_v
:=
\frac{\sum_{i=1}^m \sigma_i \widetilde s_v^{(i)}}
{\left\|\sum_{i=1}^m \sigma_i \widetilde s_v^{(i)}\right\|_2}.
\]
Operationally, synchronization converts raw pathwise propagation into a first approximation of a semantic section in the sense of Definition~\ref{def:semantic-section}.
Without this step, the procedure would merely record path-dependent transport traces; with it, the procedure attempts to infer a single local representative that best reconciles the incoming routes.

\subsection{Defect-based pruning and support extraction}

Naive propagation typically produces one of two degenerate outcomes:
either it collapses to tiny supports, or it diffuses across large connected regions of the graph.
To isolate meaningful candidates, we prune by local defect after synchronization.

For a synchronized candidate section $s$ supported on an induced connected subgraph $H\subseteq G$, we compute two defect families.

\paragraph{Edge defect.}
For each oriented edge $u\to v$ in $H$,
\[
\delta_{vu}(s)
:=
\inf_{\sigma\in\{\pm 1\}}
\left\|
\frac{T_{vu}s_u}{\|T_{vu}s_u\|_2}
-
\sigma\frac{s_v}{\|s_v\|_2}
\right\|_2.
\]

\paragraph{Path defect.}
For each chart $c\in V_H$, we compare the synchronized representative $s_c$ with the representative obtained by transporting the seed $s_r$ along the retained parent path from $r$ to $c$.
This measures how far the synchronized section deviates from the original seeded transport trace.

Charts and edges whose defects exceed prescribed thresholds are removed, and the support is restricted to the connected component containing the seed.
The result is a maximal defect-controlled support around the seed.
Empirically, this pruning stage is what separates singleton collapse, diffuse propagation, and nontrivial coherent supports.

\subsection{Cycle-aware section taxonomy}

Section~\ref{sec:semantic-sections} showed that tree-supported and cycle-supported sections must be treated differently.
Accordingly, we classify each discovered candidate using both geometric defect statistics and the cycle structure of its support graph.

For a connected support subgraph $H$, let
\[
\beta_1(H)=|E_H|-|V_H|+1
\]
denote its first Betti number, i.e.\ the number of independent cycles in the support.

\begin{definition}[Empirical section taxonomy]
A discovered section is classified as follows:
\begin{enumerate}[leftmargin=2em]
\item \textbf{Globalizable} if $\beta_1(H)\ge 1$, the mean edge defect is small, and the mean loop holonomy defect is small.
\item \textbf{Twisted} if $\beta_1(H)\ge 1$, the mean edge defect is small, but the mean loop holonomy defect is sufficiently large.
\item \textbf{Tree-local} if $\beta_1(H)=0$ and the mean edge defect is small.
\item \textbf{Fragile} otherwise.
\end{enumerate}
\end{definition}

This taxonomy is not merely a plotting convenience.
It operationalizes the conceptual distinctions established in Section~\ref{sec:semantic-sections}:
tree-local sections provide only pathwise evidence, while globalizable and twisted sections represent genuinely local-to-global semantic structure on cyclic supports.

\subsection{Deduplication of rediscovered sections}

Many seed chart/axis pairs rediscover approximately the same semantic section.
To avoid overstating multiplicity, we deduplicate discovered candidates by support overlap and geometric similarity.

Two discovered sections are merged when:
\begin{enumerate}[leftmargin=2em]
\item they have the same empirical section type,
\item their support sets have sufficiently high Jaccard overlap, and
\item their summary-level geometric statistics (edge defect, alignment, and, when applicable, holonomy) are sufficiently close.
\end{enumerate}

Each resulting cluster is represented by its strongest member, ranked by larger support, larger cycle count, lower defect, and higher alignment.
The deduplicated census therefore estimates the number of \emph{unique} semantic sections rather than the number of successful seed initializations.
This distinction is crucial: without deduplication, one risks confusing seed redundancy with genuine ontology-level structure.

\subsection{Certification viewpoint}

The discovery procedure is intentionally conservative.
A candidate is not treated as meaningful merely because it can be transported from a seed.
Instead, it must survive synchronization, defect-based pruning, and cycle-aware classification.
In this sense, the pipeline is a \emph{certification} procedure:
the strongest objects are those that remain coherent after being forced to confront independent loop constraints.

\paragraph{Interpretation.}
The pipeline operationalizes the central claim of this paper.
A feature is not discovered by fitting one global vector.
Rather, one proposes a local representative, propagates it through the atlas, synchronizes it across overlaps, prunes by defect, and then asks what kind of section remains.
The output is therefore not a single undifferentiated list of ``features,'' but a structured family of semantic sections whose mode of coherence carries the interpretive content.

\subsection{Empirical protocol for section discovery}
\label{sec:section-empirical-protocol}

We apply this discovery-and-certification pipeline to the frozen atlas constructions of \cite{javidnia2026gaugetheory}.
For each model, we begin from a fixed chart graph $G=(V,E)$, local orthonormal chart bases $\{B_c\}_{c\in V}$, and edge-level transport objects derived from overlap statistics.
All experiments in the present paper use layer-$16$ atlases for \texttt{Llama~3.2~3B~Instruct} \cite{grattafiori2024llama3,meta2024llama32card}, \texttt{Qwen~2.5~3B~Instruct} \cite{qwen2024qwen25,qwen2024qwen25card}, and \texttt{Gemma~2~2B~IT} \cite{gemmateam2024gemma2,google2024gemma2card} unless otherwise stated.

\paragraph{Seed set.}
For each model, discovery is initialized from a grid of seed chart/axis pairs.
In the reported layer-$16$ experiments we use the first $16$ seed charts and the first $32$ seed axes per chart, yielding $512$ raw seeded discoveries per model before pruning and deduplication.
This design is intentionally modest:
the goal is not exhaustive enumeration, but a controlled cross-model census of section types under a fixed discovery budget.

\paragraph{Propagation and synchronization settings.}
Discovery uses bounded-depth propagation with maximum depth $8$.
Transitions are filtered by overlap and transport-conditioning statistics inherited from the atlas construction; in the reported runs we require minimum overlap size $256$ and use the default conditioning threshold $s_{\min}=0$.
Synchronization is performed for two passes over multiply reached charts.
These settings allow nontrivial support growth while remaining conservative about weak overlaps.

\paragraph{Defect-based pruning.}
After synchronization, candidates are pruned by local defect.
The reported runs use a maximum admissible edge defect of $0.25$ and a maximum admissible path defect of $0.25$ during support extraction.
These pruning thresholds are stricter than the initial propagation rule, so that the retained support represents a coherent core rather than a diffuse propagation cloud.

\paragraph{Cycle-aware taxonomy thresholds.}
After support extraction, each candidate is classified using the cycle-aware taxonomy above.
The reported counts use the following thresholds:
\begin{itemize}[leftmargin=2em]
\item \textbf{Globalizable:} support size at least $6$, mean edge defect at most $0.16$, mean holonomy defect at most $0.35$, and at least one independent cycle.
\item \textbf{Twisted:} support size at least $4$, mean edge defect at most $0.18$, mean holonomy defect at least $0.35$, and at least one independent cycle.
\item \textbf{Tree-local:} support size at least $4$, mean edge defect at most $0.18$, and zero cycle count.
\item \textbf{Fragile:} all remaining cases.
\end{itemize}
These thresholds are operational rather than ontological.
They are a first operational instantiation of the conceptual distinctions introduced in Section~\ref{sec:semantic-sections}.

\paragraph{Deduplication protocol.}
To estimate the number of unique section objects rather than the number of successful seeds, we deduplicate discovered sections by merging candidates that satisfy three conditions:
\begin{enumerate}[leftmargin=2em]
\item identical section type,
\item support-set Jaccard overlap at least $0.75$, and
\item summary-level geometric discrepancy at most $0.05$.
\end{enumerate}
The representative of each deduplicated cluster is chosen by preferring larger support, larger cycle count, lower edge defect, and higher alignment.

\paragraph{Why this protocol is sufficient for a first census.}
The experiments provide a systematic empirical test of the semantic-section ontology, though they do not exhaust the full model and task space.
The central question is whether the same discovery-and-certification pipeline yields nontrivial section types across models after deduplication.
As the experiments below show, the answer is yes:
the resulting census contains globalizable, twisted, and tree-local regimes, and these patterns remain informative under threshold and layer variation.
The full discovery-and-certification settings are summarized in Table~\ref{tab:section_protocol}.

\begin{table}[t]
\centering
\caption{Section-discovery protocol used for the cross-model section census. All runs use the same discovery and classification settings unless otherwise stated.}
\label{tab:section_protocol}
\begin{tabular}{ll}
\toprule
Quantity & Value \\
\midrule
Seed charts per model & $16$ \\
Seed axes per chart & $32$ \\
Raw seeded discoveries per model & $512$ \\
Maximum propagation depth & $8$ \\
Minimum overlap size & $256$ \\
Transport conditioning threshold & $s_{\min}=0$ \\
Synchronization passes & $2$ \\
Pruning: max edge defect & $0.25$ \\
Pruning: max path defect & $0.25$ \\
Globalizable: min support & $6$ \\
Globalizable: max mean edge defect & $0.16$ \\
Globalizable: max mean holonomy defect & $0.35$ \\
Twisted: min support & $4$ \\
Twisted: max mean edge defect & $0.18$ \\
Twisted: min mean holonomy defect & $0.35$ \\
Tree-local: min support & $4$ \\
Tree-local: max mean edge defect & $0.18$ \\
Deduplication Jaccard threshold & $0.75$ \\
Deduplication metric-gap threshold & $0.05$ \\
\bottomrule
\end{tabular}
\end{table}

\section{Experiments}
\label{sec:experiments}

This section tests the central empirical prediction of the semantic-section view.
If semantic structure in obstructed representation spaces is atlas-native rather than globally dictionary-based, then frozen models should not exhibit one uniform feature regime.
Instead, one should observe a mixture of section types:
some candidate meanings should globalize across cyclic supports with low defect,
some should remain only tree-local,
and some should be locally coherent while becoming holonomy-twisted.
The experiments below test this prediction in stages.
We first establish a cross-model census of section types, then examine their geometric profile and concrete exemplars, and finally show that semantic identity is recovered by sections even when raw global-vector similarity fails.

\subsection{Cross-model census of semantic section types}
\label{sec:section-census}

We begin by asking whether the semantic-section taxonomy appears across multiple model families.
Using the discovery-and-certification pipeline of Section~\ref{sec:discovery-certification}, we classify each deduplicated candidate into one of four empirical types:
\emph{globalizable} (cycle-supported, low-defect, low-holonomy sections),
\emph{twisted} (cycle-supported but holonomy-obstructed sections),
\emph{tree-local} (low-defect sections on acyclic supports),
and \emph{fragile} (all remaining candidates).
Because many seed/axis initializations rediscover the same object, all counts reported below are post-deduplication counts.

\begin{table}[t]
\centering
\caption{Deduplicated cross-model census of semantic section types at layer 16. Counts are reported after merging near-duplicate sections discovered from multiple seeds.}
\label{tab:section_census_cross_model}
\begin{tabular}{lrrrrr}
\toprule
Model & Globalizable & Twisted & Tree-local & Fragile & Total unique \\
\midrule
\texttt{Llama~3.2~3B~Instruct} & 1 & 1 & 2 & 32 & 36 \\
\texttt{Qwen~2.5~3B~Instruct}  & 1 & 2 & 5 & 16 & 24 \\
\texttt{Gemma~2~2B~IT}         & 14 & 4 & 8 & 40 & 66 \\
\bottomrule
\end{tabular}
\end{table}

Table~\ref{tab:section_census_cross_model} shows that the semantic-section taxonomy is nontrivial across all three model families.
Most discovered candidates are fragile, but each model also exhibits a smaller family of coherent section types.
Llama yields a sparse but nondegenerate regime, with one globalizable section, one twisted section, and two tree-local sections.
Qwen also retains nontrivial cycle-supported structure after deduplication, but is dominated by tree-local sections, indicating that in this atlas construction many coherent candidates remain pathwise rather than strongly cycle-certified.
Gemma exhibits the richest section structure, with fourteen globalizable sections, four twisted sections, and eight tree-local sections.

The cross-model differences are themselves informative.
Qwen demonstrates that tree-local coherence can be abundant even when cycle-supported sections are relatively scarce, which reinforces the distinction between weak local coherence and stronger atlas-level semantic globalization.
By contrast, Gemma exhibits a much larger family of cycle-supported globalizable sections, indicating that the same pipeline can recover substantial local-to-global semantic structure in at least some model families.
Llama lies between these regimes: it supports the full taxonomy, but sparsely.
Taken together, these results support the view that semantic sections are neither a one-model artifact nor a trivial consequence of the discovery pipeline, but a model-dependent ontology for feature-like structure in obstructed representation spaces.

\subsection{Type-wise geometric profile of deduplicated sections}
\label{sec:section-type-profile}

The census in Table~\ref{tab:section_census_cross_model} counts how many unique sections of each type appear.
We next ask what geometric profile these types actually exhibit.
For each model and section type, we report the number of unique sections together with the mean support size, mean cycle count, mean edge defect, and mean alignment.

\begin{table*}[t]
\centering
\caption{Type-wise geometric profile of deduplicated semantic sections at layer 16. Support is the number of charts in the section support; cycles counts independent cycles in the induced support subgraph. Lower edge defect and higher alignment indicate stronger transport compatibility.}
\label{tab:section_type_profile}
\resizebox{\textwidth}{!}{
\begin{tabular}{llrrrrr}
\toprule
Model & Section type & Count & Mean support & Mean cycles & Mean edge defect & Mean alignment \\
\midrule
\texttt{Llama~3.2~3B~Instruct} & Globalizable & 1  & 6.0000  & 1.0000 & 0.1367 & 0.9892 \\
\texttt{Llama~3.2~3B~Instruct} & Twisted      & 1  & 4.0000  & 1.0000 & 0.1232 & 0.9908 \\
\texttt{Llama~3.2~3B~Instruct} & Tree-local   & 2  & 4.5000  & 0.0000 & 0.1486 & 0.9887 \\
\texttt{Llama~3.2~3B~Instruct} & Fragile      & 32 & 1.9375  & 0.0938 & 0.0720 & 0.9946 \\
\addlinespace
\texttt{Qwen~2.5~3B~Instruct}  & Globalizable & 1  & 9.0000  & 1.0000 & 0.1106 & 0.9915 \\
\texttt{Qwen~2.5~3B~Instruct}  & Twisted      & 2  & 8.0000  & 1.5000 & 0.1217 & 0.9902 \\
\texttt{Qwen~2.5~3B~Instruct}  & Tree-local   & 5  & 7.2000  & 0.0000 & 0.0569 & 0.9973 \\
\texttt{Qwen~2.5~3B~Instruct}  & Fragile      & 16 & 1.8125  & 0.0000 & 0.0311 & 0.9987 \\
\addlinespace
\texttt{Gemma~2~2B~IT}         & Globalizable & 14 & 12.5000 & 4.2143 & 0.1082 & 0.9928 \\
\texttt{Gemma~2~2B~IT}         & Twisted      & 4  & 9.2500  & 2.0000 & 0.1152 & 0.9909 \\
\texttt{Gemma~2~2B~IT}         & Tree-local   & 8  & 5.1250  & 0.0000 & 0.0961 & 0.9946 \\
\texttt{Gemma~2~2B~IT}         & Fragile      & 40 & 2.3750  & 0.2250 & 0.0765 & 0.9950 \\
\bottomrule
\end{tabular}}
\end{table*}

Table~\ref{tab:section_type_profile} sharpens the interpretation of the section census.
Globalizable and twisted sections are both cycle-supported, but they differ in how holonomy behaves: globalizable sections maintain low defect and low holonomy on cyclic supports, whereas twisted sections remain locally coherent while failing to globalize consistently around loops.
By contrast, tree-local sections have zero cycle count by construction and therefore provide only pathwise coherence.
This distinction is especially important for Qwen, where many deduplicated sections are tree-local rather than genuinely cycle-supported, and for Gemma, where the larger family of globalizable sections indicates richer local-to-global semantic structure under identical discovery settings.

The importance of cycle support is conceptual as well as numerical.
A tree-supported section can always be propagated from a seed without confronting an independent loop constraint, so tree-local coherence is weaker evidence than cycle-supported globalization.
For this reason, the strongest empirical support for the semantic-section ontology comes from the existence of low-defect cyclic sections and from the separation between globalizable and twisted regimes.

\subsection{Representative semantic sections across regimes}
\label{sec:section-exemplars}

To make the ontology concrete, we inspect representative deduplicated sections from the main regimes uncovered by the cross-model census.

We focus on eight exemplars:
two from Llama (one globalizable and one twisted),
three from Qwen (one tree-local control, one globalizable, and one twisted),
and three from Gemma (two globalizable and one twisted).

The goal is not to claim that these are the only important sections, but to show that the discovered objects already exhibit the qualitative behaviors predicted by the theory.
Their geometric profiles are summarized in Table~\ref{tab:section_exemplars}.

\paragraph{Llama: a sparse but complete taxonomy.}
Llama exhibits the full section taxonomy, though sparsely.
Its strongest globalizable exemplar is the unique deduplicated globalizable section discovered from seed $(12,11)$.
It has support size $6$, one independent cycle, mean edge defect $0.1367$, mean alignment $0.9892$, and mean holonomy defect $0.3199$.
Its support spans charts $\{12,26,38,78,108,111\}$.
This is the clearest Llama example of a section that remains coherent on a cyclic support while keeping holonomy below the globalization threshold.

The strongest Llama twisted exemplar is the deduplicated section represented by seed $(15,0)$, with support size $4$, one cycle, mean edge defect $0.1232$, mean alignment $0.9908$, and mean holonomy defect $0.3750$.
Its duplicate cluster merges three seed discoveries, namely $(15,0)$, $(15,7)$, and $(15,8)$, over the chart set $\{15,27,30,101\}$.
This is a useful case because the section remains locally coherent and cycle-supported, but crosses the holonomy threshold and therefore fails to globalize cleanly.
Together, the two Llama exemplars already realize the conceptual distinction between low-holonomy globalization and holonomy-twisted obstruction.

\paragraph{Qwen: tree-local coherence as a control regime.}
Qwen is dominated by tree-local rather than cycle-supported sections under the current discovery protocol.
The clearest control-case exemplar is the large deduplicated tree-local section represented by seed $(5,0)$.
It has support size $15$, zero cycles, mean edge defect approximately $4.2\times 10^{-7}$, and mean alignment numerically equal to $1$.
Its duplicate cluster merges $192$ seed discoveries into the same tree-supported object over the chart set
$\{0,3,5,8,10,15,28,43,72,77,79,83,93,101,115\}$.
This section is important precisely because it is \emph{too coherent}: since its support is tree-like, the section never encounters an independent loop constraint, and pathwise propagation becomes nearly tautological.
Qwen therefore serves as a useful control regime showing that tree-local coherence is much weaker evidence than cycle-supported globalization.

Qwen nevertheless retains a nontrivial cycle-supported signal.
Its strongest deduplicated globalizable exemplar is the section seeded at $(14,10)$, with support size $9$, one cycle, mean edge defect $0.1106$, mean alignment $0.9915$, and mean holonomy defect $0.3363$ over charts $\{1,13,14,52,63,64,96,97,126\}$.
Its strongest twisted exemplar is seeded at $(14,17)$, with support size $9$, two cycles, mean edge defect $0.1218$, mean alignment $0.9898$, and mean holonomy defect $0.3954$ over charts $\{1,13,14,21,52,63,96,97,126\}$.
Thus Qwen is not a negative case for the semantic-section view; rather, it is a model in which tree-local sections dominate while genuinely cycle-supported sections are rarer.

\paragraph{Gemma: rich cycle-supported semantic structure.}
Gemma provides the strongest evidence for the semantic-section ontology in the study.
Its best globalizable exemplar is the deduplicated section seeded at $(14,4)$, with support size $23$, nine cycles, mean edge defect $0.0960$, mean alignment $0.9946$, and mean holonomy defect $0.2245$.
The support spans the cyclic subsystem
$\{3,4,9,13,14,15,16,17,32,33,34,36,61,64,70,86,87,90,95,106,107,111,116\}$.
This is exactly the kind of object that the section view predicts and the global-feature view cannot naturally express: a high-support, low-defect, low-holonomy semantic object distributed across many local charts.

A second Gemma globalizable exemplar is the section seeded at $(5,1)$, with support size $21$, eight cycles, mean edge defect $0.0954$, mean alignment $0.9945$, and mean holonomy defect $0.2146$.
Its support covers charts
$\{0,5,10,18,24,26,35,40,47,49,50,57,59,65,68,89,91,100,102,126,127\}$.
The fact that Gemma contains multiple large, low-defect, low-holonomy cyclic sections strengthens the claim that the ontology is not built around one isolated anecdote.

The strongest Gemma twisted exemplar is the deduplicated section represented by seed $(13,9)$.
It has support size $14$, four cycles, mean edge defect $0.1071$, mean alignment $0.9920$, and mean holonomy defect $0.3971$.
Its duplicate cluster merges three seed discoveries, $(3,1)$, $(3,7)$, and $(13,9)$, over charts
$\{3,4,13,33,34,36,64,70,86,87,95,106,111,116\}$.
This is the clearest illustration in the study of a section that is both robust and obstructed:
its local compatibility is strong, but loop transport prevents a single globally consistent continuation.

\paragraph{Interpretation.}
These exemplars clarify the ontology behind the section census.
The Llama examples show that even a sparse model regime can realize both globalization and twisting.
The Qwen control case shows that tree-local propagation can produce highly coherent objects without yet confronting genuine local-to-global consistency.
The Gemma examples show that large, low-defect, cycle-supported sections can exist in practice and that twisted sections appear as nearby but qualitatively distinct objects.
Taken together, these case studies make the central conceptual point concrete:
in obstructed representation spaces, a feature is better modeled as a transport-compatible local section than as a single global direction or dictionary atom.

\begin{table*}[t]
\centering
\caption{Representative deduplicated semantic sections used as exemplars in Section~\ref{sec:section-exemplars}, illustrating globalizable, twisted, and tree-local regimes across the three model families.}
\label{tab:section_exemplars}
\resizebox{\textwidth}{!}{
\begin{tabular}{llrrrrrl}
\toprule
Model & Type & Seed & Support & Cycles & Mean edge defect & Mean holonomy defect & Representative chart support \\
\midrule
\texttt{Llama~3.2~3B~Instruct} & Globalizable & $(12,11)$ & 6  & 1 & 0.1367 & 0.3199 & $\{12,26,38,78,108,111\}$ \\
\texttt{Llama~3.2~3B~Instruct} & Twisted      & $(15,0)$  & 4  & 1 & 0.1232 & 0.3750 & $\{15,27,30,101\}$ \\
\texttt{Qwen~2.5~3B~Instruct}  & Tree-local   & $(5,0)$   & 15 & 0 & $\approx 0$ & \text{---}    & $\{0,3,5,8,10,15,28,43,72,77,79,83,93,101,115\}$ \\
\texttt{Qwen~2.5~3B~Instruct}  & Globalizable & $(14,10)$ & 9  & 1 & 0.1106 & 0.3363 & $\{1,13,14,52,63,64,96,97,126\}$ \\
\texttt{Qwen~2.5~3B~Instruct}  & Twisted      & $(14,17)$ & 9  & 2 & 0.1218 & 0.3954 & $\{1,13,14,21,52,63,96,97,126\}$ \\
\texttt{Gemma~2~2B~IT}         & Globalizable & $(14,4)$  & 23 & 9 & 0.0960 & 0.2245 & $\{3,4,9,13,14,15,\dots,107,111,116\}$ \\
\texttt{Gemma~2~2B~IT}         & Globalizable & $(5,1)$   & 21 & 8 & 0.0954 & 0.2146 & $\{0,5,10,18,24,26,\dots,102,126,127\}$ \\
\texttt{Gemma~2~2B~IT}         & Twisted      & $(13,9)$  & 14 & 4 & 0.1071 & 0.3971 & $\{3,4,13,33,34,36,64,70,86,87,95,106,111,116\}$ \\
\bottomrule
\end{tabular}}
\end{table*}

\subsection{Semantic sections are not single global vectors}
\label{sec:sections-not-global-vectors}

The semantic-section view predicts that a coherent semantic object need not appear as one globally shared vector across charts.
To test this directly, we compare atlas-native section coherence with raw cross-chart similarity of the local representatives themselves.
For the representative exemplars in Section~\ref{sec:section-exemplars}, together with a small number of additional tree-local baseline sections, we measure pairwise signed cosine similarity between all local representatives in each support, alongside the section's mean edge defect, mean path defect, and, when applicable, mean holonomy defect. This comparison now includes Qwen exemplars as well as Llama and Gemma, allowing us to test whether the same mismatch between section coherence and raw global-vector similarity appears in the tree-local control regime and in Qwen's rarer cycle-supported sections.

\begin{table*}[t]
\centering
\caption{Semantic-section coherence versus raw cross-chart vector similarity for the representative exemplars in Section~\ref{sec:section-exemplars}, together with a small number of additional tree-local baseline sections. Even when a candidate object is section-coherent, its raw local representatives are not close as a single shared vector across charts.}
\label{tab:section_vs_raw_similarity}
\resizebox{\textwidth}{!}{
\begin{tabular}{llrrrrrrrrrr}
\toprule
Model & Type & Seed & Support & Cycles & Mean edge defect & Mean path defect & Mean holonomy defect & Pairwise mean cosine & Pairwise min cosine & Fraction $\ge 0.5$ & Fraction $\ge 0.7$ \\
\midrule
\texttt{Llama~3.2~3B~Instruct}, layer 16 & Globalizable & $(12,11)$ & 6  & 1 & 0.1367 & 0.1003 & 0.3199 & 0.2139 & 0.0796 & 0.0000 & 0.0000 \\
\texttt{Llama~3.2~3B~Instruct}, layer 16 & Twisted      & $(15,0)$  & 4  & 1 & 0.1232 & 0.0832 & 0.3750 & 0.2475 & 0.0615 & 0.1667 & 0.0000 \\
\texttt{Llama~3.2~3B~Instruct}, layer 16 & Tree-local   & $(11,9)$  & 5  & 0 & 0.1419 & 0.1366 & \text{---}    & 0.3028 & 0.0701 & 0.3000 & 0.0000 \\
\texttt{Llama~3.2~3B~Instruct}, layer 16 & Tree-local   & $(0,5)$   & 4  & 0 & 0.1553 & 0.1294 & \text{---}    & 0.1246 & 0.0284 & 0.1667 & 0.0000 \\
\addlinespace
\texttt{Qwen~2.5~3B~Instruct}, layer 16  & Tree-local   & $(5,0)$   & 15 & 0 & 0.0000 & 0.0000 & \text{---}    & 0.3483 & 0.0055 & 0.3143 & 0.0571 \\
\texttt{Qwen~2.5~3B~Instruct}, layer 16  & Globalizable & $(14,10)$ & 9  & 1 & 0.1106 & 0.0942 & 0.3363 & 0.1814 & 0.0069 & 0.0000 & 0.0000 \\
\texttt{Qwen~2.5~3B~Instruct}, layer 16  & Twisted      & $(14,17)$ & 9  & 2 & 0.1218 & 0.0995 & 0.3954 & 0.1536 & 0.0005 & 0.0000 & 0.0000 \\
\addlinespace
\texttt{Gemma~2~2B~IT}, layer 16 & Globalizable & $(14,4)$  & 23 & 9 & 0.0960 & 0.1097 & 0.2245 & 0.2309 & 0.0037 & 0.0474 & 0.0000 \\
\texttt{Gemma~2~2B~IT}, layer 16 & Globalizable & $(5,1)$   & 21 & 8 & 0.0954 & 0.0884 & 0.2146 & 0.1921 & 0.0006 & 0.0429 & 0.0000 \\
\texttt{Gemma~2~2B~IT}, layer 16 & Twisted      & $(13,9)$  & 14 & 4 & 0.1071 & 0.0740 & 0.3971 & 0.1565 & 0.0100 & 0.0110 & 0.0000 \\
\texttt{Gemma~2~2B~IT}, layer 16 & Twisted      & $(10,10)$ & 10 & 1 & 0.0926 & 0.1088 & 0.3700 & 0.1555 & 0.0036 & 0.0000 & 0.0000 \\
\texttt{Gemma~2~2B~IT}, layer 16 & Tree-local   & $(3,13)$  & 8  & 0 & 0.0905 & 0.0938 & \text{---}    & 0.1693 & 0.0074 & 0.0000 & 0.0000 \\
\texttt{Gemma~2~2B~IT}, layer 16 & Tree-local   & $(13,2)$  & 5  & 0 & 0.0495 & 0.0502 & \text{---}    & 0.1974 & 0.0248 & 0.1000 & 0.0000 \\
\bottomrule
\end{tabular}}
\end{table*}

Table~\ref{tab:section_vs_raw_similarity} provides the clearest direct evidence that semantic sections are not reducible to single global feature vectors.
Across Llama, Qwen, and Gemma exemplars, atlas-native coherence remains strong: mean edge defects and mean path defects are low, and globalizable sections are cleanly separated from twisted sections by loop holonomy.
However, the raw local representatives of the same discovered object are not close as one shared vector across charts.
For the Llama globalizable section seeded at $(12,11)$, the pairwise mean signed cosine is only $0.2139$, the minimum pairwise cosine is $0.0796$, and no chart pair exceeds cosine threshold $0.5$.
For the Qwen globalizable and twisted sections seeded at $(14,10)$ and $(14,17)$, the pairwise mean signed cosine is only $0.1814$ and $0.1536$, respectively, with no chart pair exceeding $0.5$ in either case.
Even the large Qwen tree-local control section seeded at $(5,0)$, despite near-perfect pathwise coherence, has pairwise mean cosine only $0.3483$, with only $31.43\%$ of chart pairs above $0.5$ and $5.71\%$ above $0.7$.
For the Gemma globalizable sections seeded at $(14,4)$ and $(5,1)$, the pairwise mean signed cosine is only $0.2309$ and $0.1921$, respectively, with fractions above $0.5$ of only $0.0474$ and $0.0429$, and no chart pair exceeds $0.7$.
Thus these objects are coherent as transport-compatible local families, but not as single globally shared directions.
A raw global-feature baseline based on cross-context vector similarity would therefore fail to identify them as one object, whereas the semantic-section formalism recovers them naturally.

\subsection{Semantic identity is recovered by sections, not by raw global similarity}
\label{sec:section-identity-recovery}

The raw-similarity analysis above is negative evidence against the global-vector ontology:
it shows that coherent semantic sections do not look like one shared vector across charts.
We now turn that observation into a direct recovery test.
For each exemplar section, we treat every pair of charts in its support as a \emph{true same-object pair}, since all of them belong to one discovered semantic section.
We then compare two identity criteria:
\begin{enumerate}[leftmargin=2em]
\item \textbf{Section criterion:} two local representatives are identified as the same semantic object if they belong to the same discovered semantic section.
\item \textbf{Raw global-similarity baseline:} two local representatives are identified as the same object only if their raw signed cosine similarity exceeds a fixed threshold.
\end{enumerate}
This yields a direct question:
how much of the within-section semantic identity can be recovered by raw vector similarity alone across Llama, Qwen, and Gemma exemplars?

\begin{table*}[t]
\centering
\caption{Semantic identity recovery: sections versus raw global similarity. The section criterion recovers all within-section pairs by construction, while raw signed-cosine similarity recovers only a small fraction of them.}
\label{tab:section_identity_recovery}
\resizebox{\textwidth}{!}{
\begin{tabular}{llrrrrrrrrrr}
\toprule
Model & Type & Seed & Support & Cycles & Pairs & Section recovery & Raw @ $0.3$ & Raw @ $0.5$ & Raw @ $0.7$ & Mean cosine & Min cosine \\
\midrule
\texttt{Llama~3.2~3B~Instruct}, layer 16 & Globalizable & $(12,11)$ & 6  & 1 & 15  & 1.0000 & 0.2667 & 0.0000 & 0.0000 & 0.2139 & 0.0796 \\
\texttt{Llama~3.2~3B~Instruct}, layer 16 & Twisted      & $(15,0)$  & 4  & 1 & 6   & 1.0000 & 0.3333 & 0.1667 & 0.0000 & 0.2475 & 0.0615 \\
\texttt{Llama~3.2~3B~Instruct}, layer 16 & Tree-local   & $(11,9)$  & 5  & 0 & 10  & 1.0000 & 0.5000 & 0.3000 & 0.0000 & 0.3028 & 0.0701 \\
\texttt{Llama~3.2~3B~Instruct}, layer 16 & Tree-local   & $(0,5)$   & 4  & 0 & 6   & 1.0000 & 0.1667 & 0.1667 & 0.0000 & 0.1246 & 0.0284 \\
\addlinespace
\texttt{Qwen~2.5~3B~Instruct}, layer 16  & Tree-local   & $(5,0)$   & 15 & 0 & 105 & 1.0000 & 0.5619 & 0.3143 & 0.0571 & 0.3483 & 0.0055 \\
\texttt{Qwen~2.5~3B~Instruct}, layer 16  & Globalizable & $(14,10)$ & 9  & 1 & 36  & 1.0000 & 0.2500 & 0.0000 & 0.0000 & 0.1814 & 0.0069 \\
\texttt{Qwen~2.5~3B~Instruct}, layer 16  & Twisted      & $(14,17)$ & 9  & 2 & 36  & 1.0000 & 0.1389 & 0.0000 & 0.0000 & 0.1536 & 0.0005 \\
\addlinespace
\texttt{Gemma~2~2B~IT}, layer 16 & Globalizable & $(14,4)$  & 23 & 9 & 253 & 1.0000 & 0.3083 & 0.0474 & 0.0000 & 0.2309 & 0.0037 \\
\texttt{Gemma~2~2B~IT}, layer 16 & Globalizable & $(5,1)$   & 21 & 8 & 210 & 1.0000 & 0.1952 & 0.0429 & 0.0000 & 0.1921 & 0.0006 \\
\texttt{Gemma~2~2B~IT}, layer 16 & Twisted      & $(13,9)$  & 14 & 4 & 91  & 1.0000 & 0.1209 & 0.0110 & 0.0000 & 0.1565 & 0.0100 \\
\texttt{Gemma~2~2B~IT}, layer 16 & Twisted      & $(10,10)$ & 10 & 1 & 45  & 1.0000 & 0.1778 & 0.0000 & 0.0000 & 0.1555 & 0.0036 \\
\texttt{Gemma~2~2B~IT}, layer 16 & Tree-local   & $(3,13)$  & 8  & 0 & 28  & 1.0000 & 0.1071 & 0.0000 & 0.0000 & 0.1693 & 0.0074 \\
\texttt{Gemma~2~2B~IT}, layer 16 & Tree-local   & $(13,2)$  & 5  & 0 & 10  & 1.0000 & 0.2000 & 0.1000 & 0.0000 & 0.1974 & 0.0248 \\
\bottomrule
\end{tabular}}
\end{table*}

Table~\ref{tab:section_identity_recovery} gives the strongest exemplar-level support in the paper for the semantic-section ontology.
Across all exemplars, section-based identity recovery is perfect at the support level: once a candidate object has been certified as one semantic section, all within-section pairs are recovered.
By contrast, raw global-similarity recovery is poor.
The failure is especially clear for the strongest globalizable sections.
For the Llama globalizable section seeded at $(12,11)$, raw cosine similarity recovers only $26.67\%$ of within-section pairs at threshold $0.3$, and none at threshold $0.5$.
For the Qwen globalizable section seeded at $(14,10)$, raw similarity recovers only $25.00\%$ of within-section pairs at threshold $0.3$, and none at thresholds $0.5$ or $0.7$.
For the Qwen twisted section seeded at $(14,17)$, raw similarity recovers only $13.89\%$ of within-section pairs at threshold $0.3$, and none at thresholds $0.5$ or $0.7$.
Even the Qwen tree-local control section seeded at $(5,0)$ is only partially recoverable by raw similarity, with recovery $56.19\%$ at threshold $0.3$, $31.43\%$ at threshold $0.5$, and $5.71\%$ at threshold $0.7$ despite its near-perfect pathwise coherence.
For the Gemma globalizable sections seeded at $(14,4)$ and $(5,1)$, raw similarity recovers only $30.83\%$ and $19.52\%$ of within-section pairs at threshold $0.3$, only $4.74\%$ and $4.29\%$ at threshold $0.5$, and none at threshold $0.7$.
Twisted sections are even less recoverable by raw similarity.

The interpretation is direct.
These candidate meanings are stable enough to be recovered as single section objects, but not stable enough to be recognized as one object by a global-vector identity rule.
In other words, semantic identity is preserved by atlas-native transport-compatible organization, not by raw vector equality across contexts.
A practitioner who insists on a global-feature ontology would therefore fragment one semantic object into many apparently unrelated local vectors.
The semantic-section formalism avoids this failure by treating those local representatives as one transport-compatible family.

This result upgrades the semantic-section view from a descriptive taxonomy to an operational replacement for the global-feature ontology.
The issue is not merely that global vectors are imperfect summaries.
Rather, in obstructed regimes they fail at the basic task of semantic identity recovery.
Sections succeed precisely because they provide the appropriate local-to-global object for the obstructed regimes studied here.

\subsection{Aggregate identity recovery across section types}
\label{sec:aggregate-identity-recovery}

The identity-recovery result above is based on representative exemplars.
We next ask whether the same pattern holds in aggregate across all deduplicated sections of each type that admit pairwise recovery analysis.
To preview the layerwise robustness analysis in Section~\ref{sec:section-layer-robustness}, Table~\ref{tab:aggregate_identity_recovery} also includes corresponding aggregate identity-recovery results for \texttt{Llama~3.2~3B~Instruct} at layers 12 and 20.

\begin{table*}[t]
\centering
\caption{Aggregate semantic identity recovery by section type. Across models and layers, section-based recovery is perfect on certified supports, while raw global-similarity recovery remains poor and collapses to zero at stricter thresholds.}
\label{tab:aggregate_identity_recovery}
\resizebox{\textwidth}{!}{
\begin{tabular}{llrrrrrrrr}
\toprule
Model & Type & $n$ sections & Mean support & Mean cycles & Total pairs & Section & Raw @ $0.3$ & Raw @ $0.5$ & Raw @ $0.7$ \\
\midrule
\texttt{Llama~3.2~3B~Instruct}, layer 12 & Globalizable & 1 & 8.0000 & 8.0000 & 28   & 1.0000 & 0.1071 & 0.0714 & 0.0000 \\
\texttt{Llama~3.2~3B~Instruct}, layer 12 & Tree-local   & 5 & 6.4000 & 0.0000 & 101  & 1.0000 & 0.0792 & 0.0000 & 0.0000 \\
\addlinespace
\texttt{Llama~3.2~3B~Instruct}, layer 16 & Globalizable & 1 & 6.0000 & 1.0000 & 15   & 1.0000 & 0.2667 & 0.0000 & 0.0000 \\
\texttt{Llama~3.2~3B~Instruct}, layer 16 & Twisted      & 1 & 4.0000 & 1.0000 & 6    & 1.0000 & 0.3333 & 0.1667 & 0.0000 \\
\texttt{Llama~3.2~3B~Instruct}, layer 16 & Tree-local   & 2 & 4.5000 & 0.0000 & 16   & 1.0000 & 0.3750 & 0.2500 & 0.0000 \\
\addlinespace
\texttt{Llama~3.2~3B~Instruct}, layer 20 & Tree-local   & 2 & 4.5000 & 0.0000 & 16   & 1.0000 & 0.2500 & 0.0000 & 0.0000 \\
\addlinespace
\texttt{Qwen~2.5~3B~Instruct}, layer 16 & Globalizable & 1 & 9.0000 & 1.0000 & 36  & 1.0000 & 0.2500 & 0.0000 & 0.0000 \\
\texttt{Qwen~2.5~3B~Instruct}, layer 16 & Tree-local   & 5 & 7.2000 & 0.0000 & 152 & 1.0000 & 0.4474 & 0.2171 & 0.0395 \\
\texttt{Qwen~2.5~3B~Instruct}, layer 16 & Twisted      & 2 & 8.0000 & 1.5000 & 57  & 1.0000 & 0.1228 & 0.0000 & 0.0000 \\
\addlinespace
\texttt{Gemma~2~2B~IT}, layer 16 & Globalizable & 14 & 12.5000 & 4.2143 & 1221 & 1.0000 & 0.1957 & 0.0262 & 0.0000 \\
\texttt{Gemma~2~2B~IT}, layer 16 & Twisted      & 4  & 9.2500  & 2.0000 & 174  & 1.0000 & 0.1437 & 0.0057 & 0.0000 \\
\texttt{Gemma~2~2B~IT}, layer 16 & Tree-local   & 8  & 5.1250  & 0.0000 & 90   & 1.0000 & 0.2222 & 0.0333 & 0.0000 \\
\bottomrule
\end{tabular}}
\end{table*}

\begin{figure}[t]
    \centering
    \includegraphics[width=\textwidth]{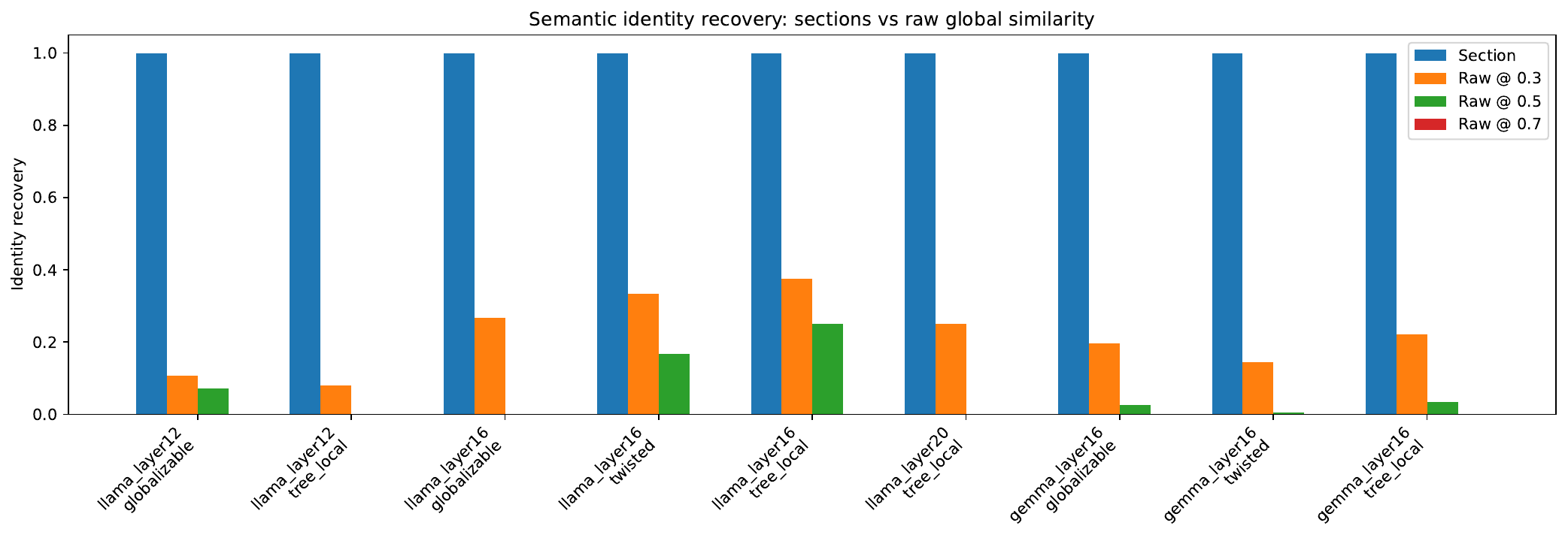}
    \caption{Aggregate semantic identity recovery across section types. Section-based recovery remains perfect on certified supports, while raw global-similarity recovery is consistently lower and typically collapses at stricter thresholds.}
    \label{fig:aggregate_identity_recovery}
\end{figure}

Table~\ref{tab:aggregate_identity_recovery} shows that the section-based identity result is not confined to a few hand-picked examples.
Across the deduplicated globalizable, twisted, and tree-local sections in Table~\ref{tab:aggregate_identity_recovery}, section-level recovery is perfect on the certified supports, whereas raw global-similarity recovery remains poor.
The strongest aggregate result appears for \texttt{Gemma~2~2B~IT}, layer 16 globalizable sections: across $14$ sections and $1221$ true within-section chart pairs, raw signed-cosine similarity recovers only $19.57\%$ of pairs at threshold $0.3$, only $2.62\%$ at threshold $0.5$, and none at threshold $0.7$.
Qwen shows the same qualitative failure in a different regime: across its single deduplicated globalizable section, raw recovery is only $25.00\%$ at threshold $0.3$ and zero at stricter thresholds; across its five deduplicated tree-local sections, raw recovery reaches $44.74\%$ at threshold $0.3$, $21.71\%$ at threshold $0.5$, and $3.95\%$ at threshold $0.7$; and across its two deduplicated twisted sections, raw recovery is only $12.28\%$ at threshold $0.3$ and zero at thresholds $0.5$ and $0.7$.
A similarly sharp failure appears for Gemma twisted sections, where raw recovery drops to $0.57\%$ at threshold $0.5$ and zero at $0.7$, and for \texttt{Llama~3.2~3B~Instruct}, layer 12 globalizable sections, where raw recovery is only $7.14\%$ at threshold $0.5$.
Thus the failure of the global-vector ontology is not anecdotal.
Even when a semantic object is coherent enough to be recovered as one section, raw cross-context vector similarity typically fragments it into apparently unrelated local vectors.
This is precisely the regime in which semantic sections are needed as the correct atlas-native notion of feature identity.
Figure~\ref{fig:aggregate_identity_recovery} visualizes the same pattern: section-based recovery remains uniformly perfect across section types, while raw-similarity recovery is consistently much lower and typically collapses at stricter thresholds.

\subsection{Layerwise robustness within Llama}
\label{sec:section-layer-robustness}

We next test whether the semantic-section picture is confined to a single depth or persists across layers.
To do so, we repeat the discovery, classification, and deduplication pipeline across multiple Llama layers under the same protocol.
In addition to the layer-$16$ results reported above, we evaluate layers $12$ and $20$ using the same seed budget, pruning thresholds, taxonomy thresholds, and deduplication rule.

\begin{table}[t]
\centering
\caption{Deduplicated semantic-section census across Llama layers.}
\label{tab:section_layer_robustness}
\begin{tabular}{lrrrrr}
\toprule
Model / layer & Globalizable & Twisted & Tree-local & Fragile & Total unique \\
\midrule
\texttt{Llama~3.2~3B~Instruct}, layer 12 & 1 & 0 & 5 & 17 & 23 \\
\texttt{Llama~3.2~3B~Instruct}, layer 16 & 1 & 1 & 2 & 32 & 36 \\
\texttt{Llama~3.2~3B~Instruct}, layer 20 & 0 & 0 & 2 & 21 & 23 \\
\bottomrule
\end{tabular}
\end{table}

Table~\ref{tab:section_layer_robustness} shows that the semantic-section picture is not tied to one exceptional layer, but it is also not uniform across depth.
Layer~12 is a weak-positive regime: it contains one deduplicated globalizable section and five tree-local sections, with no twisted sections under the default thresholds.
Layer~16 exhibits the richest Llama taxonomy, containing one globalizable section, one twisted section, and two tree-local sections.
Layer~20 is markedly weaker: under the same protocol it contains no deduplicated cycle-supported globalizable or twisted sections, and only two tree-local sections survive deduplication.

This depth variation is informative rather than problematic.
If semantic sections are the correct ontology in obstructed representation spaces, one should not expect the same balance of section types at every layer.
Instead, different depths may support different modes of local-to-global semantic organization.
The Llama results support exactly this picture:
cycle-supported semantic structure is present at intermediate depths, strongest at layer~16 for the discovery protocol used here, weaker but still nontrivial at layer~12, and substantially attenuated by layer~20.

The strongest layer-$12$ cyclic exemplar is seeded at $(5,5)$, with support size $8$, cycle count $8$, mean edge defect $0.1532$, mean alignment $0.9869$, and mean holonomy defect $0.2522$, placing it in the globalizable regime under the default thresholds.
By contrast, the strongest cyclic candidates at layer~20 remain below the globalization threshold after classification and therefore do not survive as deduplicated globalizable or twisted sections.
Thus the layerwise comparison sharpens the empirical picture:
the semantic-section ontology is not layer-specific, but its strongest cycle-supported realization is depth-dependent.

\subsection{Threshold robustness of the section census}
\label{sec:section-threshold-robustness}

Because the empirical taxonomy depends on explicit geometric thresholds, we finally test whether the qualitative section census is stable under mild threshold variation.
We sweep three parameters around their default values:
the maximum mean edge defect for globalization,
the maximum mean holonomy defect for globalization,
and the deduplication Jaccard threshold.
Concretely, we evaluate the grid
\[
\max \overline{\delta}_{\mathrm{edge}} \in \{0.14,0.16,0.18\},\qquad
\max \overline{\Delta}_{\mathrm{hol}} \in \{0.30,0.35,0.40\},\qquad
J_{\mathrm{dedup}} \in \{0.70,0.75,0.80\},
\]
yielding $27$ settings for each model/layer configuration.

\begin{table}[t]
\centering
\caption{Range of deduplicated section counts under threshold variation. Each entry reports the minimum and maximum count across the $27$ threshold settings.}
\label{tab:section_threshold_robustness}
\begin{tabular}{lccccc}
\toprule
Model / layer & Globalizable & Twisted & Tree-local & Fragile & Total unique \\
\midrule
\texttt{Llama~3.2~3B~Instruct}, layer 12 & $0$--$1$  & $0$--$0$ & $5$--$5$ & $17$--$18$ & $23$--$23$ \\
\texttt{Llama~3.2~3B~Instruct}, layer 16 & $0$--$1$  & $1$--$1$ & $2$--$2$ & $32$--$33$ & $36$--$36$ \\
\texttt{Qwen~2.5~3B~Instruct}, layer 16  & $0$--$1$  & $1$--$2$ & $4$--$5$ & $16$--$17$ & $22$--$24$ \\
\texttt{Gemma~2~2B~IT}, layer 16         & $13$--$18$ & $2$--$5$ & $8$--$8$ & $40$--$47$ & $64$--$74$ \\
\bottomrule
\end{tabular}
\end{table}

Table~\ref{tab:section_threshold_robustness} shows that the main empirical picture is stable under moderate threshold changes.
Llama layer~12 remains a weak-positive regime throughout the sweep: it consistently exhibits five tree-local sections, no twisted sections, and at most one globalizable section.
Llama layer~16 is even more stable structurally: across all $27$ settings it always contains one twisted section and two tree-local sections, while the globalizable count varies only between zero and one.
Thus the sparse but complete Llama layer~16 taxonomy is not a knife-edge artifact of one threshold choice.

Qwen layer~16 also remains qualitatively stable under the sweep.
Across all $27$ settings it retains a predominantly tree-local regime, with tree-local counts ranging from $4$ to $5$, twisted counts ranging from $1$ to $2$, and globalizable counts ranging from $0$ to $1$.
Thus threshold variation changes the precise partition between Qwen's weaker cycle-supported cases, but does not overturn the broader conclusion that Qwen is dominated by tree-local coherence with only limited cycle-supported structure.

Gemma layer~16 remains the strongest regime throughout.
Across the full sweep it always contains eight tree-local sections and multiple cycle-supported sections.
The globalizable count ranges from $13$ to $18$, while the twisted count ranges from $2$ to $5$.
Although the exact partition between globalizable and twisted sections varies numerically, the qualitative conclusion does not: Gemma layer~16 consistently exhibits the richest cycle-supported semantic structure in the study.

Taken together, these sweeps show that the semantic-section ontology is not an artifact of a single threshold setting.
What changes under threshold variation is the fine partition between neighboring coherent regimes, especially in Gemma and, more modestly, in Qwen.
What remains stable is the broader empirical pattern: Llama layer~12 is mostly tree-local, Llama layer~16 supports a sparse but nontrivial full taxonomy, Qwen layer~16 remains predominantly tree-local with limited cycle-supported structure, and Gemma layer~16 exhibits the richest family of cycle-supported globalizable sections.

\section{Conclusion}
\label{sec:conclusion}

This paper proposed an atlas-native replacement for the global-feature ontology in obstructed representation spaces.
The starting point is the observation that a candidate meaning need not assemble into one globally stable vector, even when it remains locally coherent across overlapping contexts.
In such regimes, the appropriate semantic object is not a single direction, dictionary atom, or latent coordinate, but a \emph{semantic section}: a transport-compatible family of local feature representatives defined over a context atlas.

We developed this claim at both the theoretical and empirical levels.
Theoretically, we formalized semantic sections and showed that tree-supported propagation is always pathwise realizable, whereas genuine semantic globalization on cyclic supports is controlled by cycle consistency.
This yields a natural distinction between \emph{tree-local}, \emph{globalizable}, and \emph{twisted} sections.
In particular, holonomy becomes not merely a geometric diagnostic, but an obstruction to single-valued semantic identity.

Empirically, we introduced a discovery-and-certification pipeline based on seeded propagation, synchronization across overlaps, defect-based pruning, cycle-aware taxonomy, and deduplication.
Applied to frozen LLM atlases, this pipeline recovers nontrivial populations of semantic sections across multiple model families.
The resulting section census is structurally heterogeneous rather than uniform:
Gemma exhibits the richest family of cycle-supported globalizable sections,
Llama supports a sparse but nontrivial full taxonomy,
and Qwen provides a useful control regime in which tree-local coherence dominates.
These patterns remain informative under threshold variation and, for Llama, across multiple depths.

The central empirical result is that semantic identity is not recovered by raw global-vector similarity.
Even for certified globalizable sections, raw cross-chart cosine similarity between local representatives is typically low, and raw similarity baselines recover only a small fraction of true within-section pairs, often collapsing to zero at moderate thresholds.
By contrast, section-based identity recovery is perfect on certified supports.
This shows that the failure of the global-vector ontology is not merely descriptive or aesthetic:
in obstructed regimes, it fails at the operational task of recognizing one semantic object across contexts.
Semantic sections succeed because they are the correct local-to-global object.

The broader claim of the paper is therefore ontological.
We do not argue that global features are never useful.
Rather, we argue that they are not the right primitive in the presence of atlas-level obstruction.
When local meanings do not glue into one globally stable vector, they should instead be modeled as transport-compatible local sections.
In that sense, this work extends the atlas-based obstruction framework of \cite{javidnia2026gaugetheory} from a geometric diagnosis of failure to a semantic theory of feature identity.

Several directions remain open.
First, the present experiments provide a substantial cross-model census rather than an exhaustive benchmark, and broader validation across models, layers, and tasks remains important.
Second, the current discovery pipeline is intentionally conservative, and improved synchronization, support extraction, and cycle certification may reveal richer section structure.
Third, semantic sections suggest new interpretability questions beyond discovery alone, including how section objects should be edited, composed, intervened on, or aligned with downstream behavioral semantics.

More generally, the results suggest a shift in how feature structure should be conceptualized in large neural systems.
If representation spaces are locally organized but globally obstructed, then interpretability must also become local-to-global.
The semantic-section formalism is a first step in that direction.

\subsubsection*{Broader Impact Statement}
This paper makes a foundational, methodological, and empirical contribution to mechanistic interpretability research.
It introduces semantic sections as an atlas-native replacement for the global-feature ontology in obstructed representation spaces, develops a concrete discovery-and-certification pipeline, and supports the proposal with quantitative results across multiple frozen LLM atlases.
We do not introduce a new capability model, deploy a user-facing system, or present a direct behavioral intervention on users.

The main positive impact of this work is scientific.
If some neural representations are only locally compatible and fail to admit a globally consistent feature ontology, then interpretability methods based solely on single global feature vectors may systematically misidentify or fragment semantic structure.
By proposing semantic sections as a local-to-global replacement object, this paper aims to improve the fidelity, precision, and scope of interpretability analyses in obstructed representation spaces.

At the same time, improved understanding of internal semantic organization can be used in multiple ways.
It may support beneficial applications such as auditing, debugging, and safety analysis, but it could also be used to refine model behavior, target internal mechanisms more effectively, or improve intervention strategies in ways that are not always aligned with safety or transparency goals.
The present work is upstream and foundational rather than directly deployment-oriented, but downstream uses depend on how the methods are applied.

The paper provides the first substantial empirical census of semantic sections in obstructed regimes, together with a new ontology for describing them; it is not, however, an exhaustive characterization of all model families or all forms of semantic organization.
We therefore view semantic sections as a theoretically motivated and empirically supported replacement object in the settings studied here, while broader validation across architectures, layers, datasets, and downstream tasks remains important.

\subsubsection*{Author Contributions}
Hossein Javidnia conceived the study; developed the theoretical framework; designed and implemented the experimental pipeline; conducted the experiments; analysed and interpreted the results; and wrote the manuscript.

\subsubsection*{Acknowledgments}
The author thanks the ADAPT Centre for providing access to high-performance computing (HPC) resources used for the experimental runs. The ADAPT Centre did not contribute to the study design, analysis, interpretation, or writing.

\bibliography{main}
\bibliographystyle{tmlr}

\end{document}